\def\year{2022}\relax
\documentclass[letterpaper]{article} 
\usepackage{aaai22}  
\usepackage{times}  
\usepackage{helvet}  
\usepackage{courier}  
\usepackage[hyphens]{url}  
\usepackage{graphicx} 
\usepackage{natbib}  
\usepackage{caption} 
\DeclareCaptionStyle{ruled}{labelfont=normalfont,labelsep=colon,strut=off} 
\usepackage{algorithm}
\usepackage[noend]{algorithmic}
\nocopyright
\usepackage{verbatim}

\usepackage{amsmath, amsthm, amssymb, thmtools}
\newtheorem{theorem}{Theorem}

\newtheorem{lemma}{Lemma}
\newtheorem{corollary}{Corollary}
\newtheorem{definition}{Definition}

\usepackage{xspace}
\newcommand{\asec}{$\mathsf{ASEC}$\xspace}
\newcommand{\iasec}{$\mathsf{IASEC}$\xspace}
\newcommand{\ese}{$\mathsf{ESE}$\xspace}

\title{Planning with Dynamically Estimated Action Costs}
\author{
    Eyal Weiss and Gal A. Kaminka
}
\affiliations{

	The MAVERICK Group, Bar-Ilan University, Israel\\
    \{weissey, galk\}@cs.biu.ac.il
}

\begin{document}

\maketitle

\begin{abstract}
Information about action costs is critical for real-world AI planning applications. 
Rather than rely solely on declarative action models, recent approaches also use black-box external action cost estimators, often learned from data, that are applied during the planning phase. 
These, however, can be computationally expensive, and produce uncertain values. 
In this paper we suggest a generalization of deterministic planning with action costs that allows selecting between multiple estimators for action cost, to balance computation time against bounded estimation uncertainty. 
This enables a much richer---and correspondingly more realistic---problem representation. 
Importantly, it allows planners to bound plan accuracy, thereby increasing reliability, while reducing unnecessary computational burden, which is critical for scaling to large problems. 
We introduce a search algorithm, generalizing $A^*$, that solves such planning problems, and additional algorithmic extensions. 
In addition to theoretical guarantees, extensive experiments show considerable savings in runtime compared to alternatives.
\end{abstract}

\section{Introduction}
Information about action costs is critical for real-world AI planning applications. Examples of this need range from urban traffic management~\cite{mccluskey2017embedding}, through business process management~\cite{marrella2019automated},  integration in ubiquitous computing~\cite{georgievski2016automated},  spacecraft orbit planning~\cite{surovik2015heuristic}, machine tool calibration~\cite{parkinson2012application}, to general purpose robotics~\cite{cashmore2015rosplan, karpas2020automated}, and  cybersecurity~\cite{boddy05,kichkaylo10,obes10,hoffmann12}.

Cost information can be derived from models described in an appropriate language such as PDDL~\cite{mcdermott1998pddl,fox2003pddl2}. As we discuss in the motivating examples below, this is very useful, but not sufficient for some realistic applications. Naturally, the expressiveness of the modeling language restricts the models it can describe. 

A complementary recent approach allows general domain-independent planners to be coupled with domain-specific external modules which can provide cost information not easily represented in current modeling languages. Instead, the planners utilize black-box modules that provide action effects (and costs) during the planning process~\cite{dornhege2009semantic,gregory12,frances17,allen21}. 

A key challenge for planners using external cost estimation modules is the potentially substantial computation time required for the estimation. The next section details a motivating example in transportation, but here's an example from robot task and motion planning~\cite{garrett2021integrated}. In many cases, a domain-independent task planner is unable to reason about geometric and kinematic constraints of movement in a continuous environment. It therefore checks for the existence and cost of possible motion actions by calling on a motion planner. When this happens, the cost is computed with high accuracy, but its computation---even for a single potential action---is expensive.  One of course may utilize faster methods of estimating the action costs, but at the price of greater uncertainty in the computed value. This would lead to generating plans that have much lower quality with respect to the actual costs when executed.

This paper proposes to dynamically balance estimation quality and computation time during planning. See~\cite{weiss2022position} for an extended discussion of the matter. Given a bound on the uncertainty in the estimated cost of the generated plan, the algorithms presented in this paper apply black-box cost estimators \emph{on demand}, so as to minimize computational burden, while meeting the target uncertainty bound. 
This is especially relevant in the context of planning with data-driven models, as cost estimators learned from data almost always entail uncertainty. 
Explicitly considering the accumulated effect of modeling uncertainty on plan cost accuracy, by attempting to achieve a target error bound, increases robustness of plans found by planners w.r.t. modeling errors, which contributes to overall planning reliability.
Additionally, compromising on reasonable target bounds and trying to achieve them using minimum computational (modeling) overhead is crucial for scaling to large planning problems.
We make several contributions.

First, the paper presents a framework for planning with \emph{multiple action cost estimates}, as a generalization of deterministic planning. Each ground action can be associated with multiple black-box estimators, potentially ordered by increasing accuracy and computational expense. 
This framework provides an enhanced ability to represent and solve deterministic planning problems with action costs.

Second, we present \asec, a modified $A^*$ algorithm~\cite{hart1968formal} that allows a planner to dynamically select estimators to maintain a bound on the plan cost uncertainty, while attempting to minimize estimation time. We show that \asec is sound but \textit{in}complete in general, though under some conditions, it is a proper generalization of $A^*$ and then complete.
Additionally, we describe a post-search procedure that may reduce cost uncertainty for plans which do not meet the target bound. 

Finally, we provide empirical evidence for the efficacy of ASEC and associated algorithms. We implemented \asec and the other algorithms on top of the  Fast Downward planner~\cite{helmert-jair2006}. Experiments with 6600 planning problems generated from planning competition benchmarks show \asec and 
extensions offer significant savings in runtime compared to alternatives, while often maintaining
estimated costs well within target bounds.

\section{Motivating Example}
We consider the case of transportation logistics problems as a motivating example. 
Problems in this domain can be described using three action templates: drive, pick-up and drop. A typical problem deals with the terminal locations of packages and vehicles, while minimizing overall cost. Actions of pick-up and drop have similar (fixed) costs, thus the
main challenge is minimizing the cost of drive actions.  

The ``Transport Sequential'' domain from IPC-2011~\cite{coles2012survey} describes
the drive costs using the length of the roads traveled as the 
basis.  The action template of drive is given as follows.
\begin{quote}
	\begin{scriptsize}
		\begin{verbatim}
			(:action drive
			:parameters (?v - vehicle ?l1 ?l2 - location)
			:precondition (and (at ?v ?l1) (road ?l1 ?l2))
			:effect (and (not (at ?v ?l1)) (at ?v ?l2)
			(increase (total-cost) (road-length ?l1 ?l2))))
		\end{verbatim}
	\end{scriptsize}
\end{quote}

While this is a useful crude estimate, real-world planning in this domain takes many more
factors into account. For example, the truck fleet optimization company Trucknet~\cite{trucknet} reportedly takes into account freight weight and volume, drive time duration, CO2 emissions, and insurance costs---all in addition to the road length.  Naturally,
computing the cost estimate with these factors is a significant computational task.

Ideally, planning in this domain would utilize multiple cost estimators of varying accuracy (and potential estimation runtime). For example, a fast and crude estimator based primarily on fuel consumption derived from trip length, a more accurate model based on the factors above, and an even more accurate cost predictor based on real-time and historical traffic conditions data (see the technical appendix for further illustration).  

Unfortunately, no planner today is able to reason about how to effectively utilize the different action cost estimates, minimizing planning runtime. Existing planners use a single cost estimate---whether declarative or dynamically-given---for every ground action in the planning process.  

Planners for real-world problems utilize complex action cost models. Declarative representations may describe fast estimators, yet they are not well suited to describe complex mathematical functions. Moreover, some costs are only modeled in black-box form (e.g., via a learned model). Hence, planners must integrate external sources of information about action costs.

As realistic planning problem spaces tend to be very large, it is clear that action cost estimation has to be \emph{dynamic} (namely estimates should be computed during planning). To do this efficiently, one should utilize cheaper estimators (w.r.t. runtime) where possible. 
This paper is a first attempt to rigorously address this challenge.

\section{Preliminaries}\label{sec:pre}
We make use of standard deterministic planning notation based on state-transition systems.
\begin{definition}\label{def:trans_sys}
	A \emph{state-transition system} is a 4-tuple $\Sigma = (S,A,\gamma,c)$, where
	\begin{itemize}
		\item $S$ is a finite set of system states.
		\item $A$ is a finite set of actions that the agent can perform.
		\item $\gamma : S \times A \to S$ is the state-transition function.
		\item $c : S \times A \to [0, \infty)$ is the action cost function.
	\end{itemize}
\end{definition}

Considering classical planning, we assume a finite set of propositional atoms $P = \{p_0, \dots, p_{n-1}\}$.
A (partial) state is then defined by a (partial) assignment $s : P \to \{0,1\}^n$, where $S$ is the set of all (full) states.
An action $a \in A$ is a pair of partial variable assignments $\langle pre(a), eff(a) \rangle$.
An action $a$ is applicable in state $s$ iff $pre(a)$ is consistent with $s$, and in this case $c(s,a)<\infty$.
Application of $a$ in $s$ results in the state $s'=\gamma(s,a)$ that is consistent with $eff(a)$ and agrees with $s$ w.r.t. the values of all the variables in $P$ that do not appear in $eff(a)$.

With this interpretation, we call $\Sigma$ a \emph{planning domain}.
Although Def.~\ref{def:trans_sys} specifies that the cost is a function of both action and state, in practice it is almost always assumed to be solely a function of the action. We therefore use $c(a)$ as a shorthand for the cost of a ground action.
%
A \emph{plan} in $\Sigma$ is then a finite sequence of actions $\pi = \langle a_1,...,a_n \rangle$, where its length is $|\pi|=n$, and its cost is defined to be $cost(\pi)=\sum_{i=1}^{n}c(a_i)$.

\begin{definition}\label{def:planning_prob}
	A \emph{planning problem} is a triple $\mathcal{P}=(\Sigma,s_0,S_g)$, where $\Sigma$ is a planning domain, $s_0 \in S$ is an initial state, and $S_g \subset S$ is a set of goal states. A \emph{solution} for $\mathcal{P}$ is a plan $\pi$ such that $\gamma(s_0,\pi) \in S_g$ (i.e., ordered sequential application of the actions in $\pi$ transitions the system from the initial state to a goal state). 
	A \emph{cost optimal solution} is a plan $\pi^*$ that satisfies
	\begin{equation*}
		c^* := c(\pi^*) = \min\{c(\pi')~|~\pi'~\text{is a solution for}~\mathcal{P} \}.
	\end{equation*}
\end{definition}

We consider the directed graph that corresponds to a planning domain $\Sigma$.
Let $\mathcal{G}_\Sigma = (\mathcal{V}, \mathcal{E}, \Theta_\Sigma)$ be a \emph{weighted digraph}. $\mathcal{V}$ is a set of vertices, each corresponding to a state $s\in S$. $\mathcal{E}$ denotes a set of edges, such that $e=(s,\gamma(s,a)) \in \mathcal{E}$ iff $a \in A$ is an applicable action in $s \in S$. The weight (cost) $c(e)\in \Theta_\Sigma$ of each such edge is given by $c(e)=c(a)$.

\section{Cost Uncertainty and Dynamic Estimation}\label{sec:main_results}
We propose an alternate framework for defining planning problems with action costs, which admits multiple external cost estimators. 
We then show that such problems can be solved using cost-optimal planning techniques, after performing appropriate adjustments.
Lastly, we introduce a new algorithm that solves problems from this new class and analyze its properties.

\subsection{Mathematical Framework}\label{subsec:math_frame}

We consider a planning problem as defined in Def.~\ref{def:planning_prob}, with the following modification: $c(\cdot)$ is not assumed
to be precisely known. Instead, each precise cost $c(e)\in \Theta_\Sigma$ is replaced with a finite set of estimators,
denoted $\Theta(e)$ for brevity (Definition~\ref{def:theta}):

\begin{definition}\label{def:theta}
	The \emph{set of cost estimators for a planning domain} $\Sigma$, with graph~$\mathcal{G}_\Sigma$, is $\Theta_\Sigma = \{ \Theta(e)~|~e \in \mathcal{E} \}$, where
	\begin{equation*}
		\Theta(e)=\{\theta_1^e,\dots,\theta_{k(e)}^e\}, k(e) \in \mathbb{N}.
	\end{equation*}
	For every edge $e$ and every valid~$i$,~$\theta_i^e$ satisfies the following assumptions:
	\begin{enumerate}
		\item When used, it returns lower and upper bounds for the true cost, i.e.,~$c_{min}(\theta_i^e) \leq c(e) \leq c_{max}(\theta_i^e)$.\label{assum:bounds}
		\item $0<c(e)<\infty$ implies~$c_{min}(\theta_i^e)>0$, $c_{max}(\theta_i^e) < \infty$.\label{assum:informative}
	\end{enumerate}
\end{definition}

Note that w.l.o.g. we may suppose that each estimator can only contribute information.
Formally, this is obtained from Assumption~\ref{assum:bounds} by keeping the tightest bounds found when more than one estimator $\theta_i^e$ is utilized.  

Assumption~\ref{assum:informative} means that estimators provide at least some information through both of their bounds. 
We do not enforce a limitation on the best attainable estimate for each edge cost.  Thus, an estimator $\theta_i^e$ whose $c_{min}(\theta_i^e)=c_{max}(\theta_i^e)$ is the special case of known cost $c(e)$.
Finally, to allow strong generality, we let each estimator be characterized by an unknown finite estimation time $\tau_i^e$.


We are now ready to state the resulting planning problem.


\begin{definition}\label{def:new_problem}
	A \emph{planning problem with edge cost estimation} is a 4-tuple $\mathcal{P}=(\Sigma,\Theta_\Sigma,s_0,S_g)$, where $\Sigma$ is a planning domain with cost function $c$ unknown, the set of estimators $\Theta_\Sigma$ is as described in Def.~\ref{def:theta}, $s_0\in S$ is an initial state, and $S_g\subset S$ is a set of goal states. A \emph{solution} for $\mathcal{P}$ is a plan $\pi$ such that $\gamma(s_0,\pi) \in S_g$.
	An \emph{$\epsilon$-bounded cost optimal (or simply $\epsilon$-optimal) solution} is a plan $\pi^\epsilon$ that satisfies
	\begin{equation*}
		c(\pi^\epsilon) \le c^* \times \epsilon,
	\end{equation*}
	with $c^*$ being the optimal cost and~$\epsilon \ge 1$.
\end{definition}

\begin{theorem}\label{thm:general}
	(\emph{Generality}).
	A planning problem with edge cost estimation (defined in Def.~\ref{def:new_problem}) is a generalization of a classical cost optimal planning problem (defined in Def.~\ref{def:planning_prob}).
\end{theorem}

\begin{proof}
\emph{Any} classical cost optimal planning problem can be formulated as a planning problem with edge cost estimation (Def.~\ref{def:new_problem}),  by considering the special case where each edge has one estimator (i.e.,~$k(e)=1$ for every~$e$), that returns the true cost (namely,~$c_{min}(\theta_1^e) = c(e) = c_{max}(\theta_1^e)$), with no relaxation for plan optimality (which means that~$\epsilon=1$).
\end{proof}


\subsection{Uncertainty Quantification}\label{subsec:uncer}
We provide a few more definitions necessary for stating our theoretical results.
For a given problem $\mathcal{P}$, let $\Phi \subseteq \Theta_\Sigma$ contain at least one estimator for every edge $e$ (i.e., $|\Theta(e)| \ge 1~\forall e \in \mathcal{E}$).
The \emph{plan lower bound} of $\pi = \langle a_1,...,a_n \rangle$ w.r.t. $\Phi$ is defined as
\begin{equation}\label{eq:plan_lower_bound}
c_{min}^{\Phi}(\pi):=\sum_{i=1}^{n}c_{min, \Phi}^i,
\end{equation}
where $c_{min, \Phi}^i$ corresponds to the tightest lower bound estimate of~$a_i$ obtained from $\Phi$.
The \emph{optimal plan lower bound} w.r.t. $\Phi$, denoted $c_{\Phi}^*$, is the minimal plan lower bound w.r.t. $\Phi$ over all plans for $\mathcal{P}$.
An \emph{optimal plan} w.r.t. $\Phi$ is a plan $\pi$ that satisfies $c_{min}^{\Phi}(\pi) = c_{\Phi}^*$.

We consider the consequences of edge cost uncertainty, characterized by lower and upper bounds, on the total plan cost uncertainty. 
Suppose that for a problem $\mathcal{P}$, as defined in Def.~\ref{def:new_problem}, we are given a plan~$\pi = \langle a_1,...,a_n \rangle$, 
with bounds~$c_{min}^i, c_{max}^i$ corresponding to the action~$a_i$, for every~$i \in [1,n]$.
Then the cost of~$\pi$ satisfies
\begin{equation}\label{eq:plan_bounds}
	\sum_{i=1}^{n}c_{min}^i \leq c(\pi) \leq \sum_{i=1}^{n}c_{max}^i,
\end{equation}
and its overall uncertainty can be further described by the effective ratio of the its bounds:
\begin{equation}\label{eq:eta_eff}
	\eta_{eff}(\pi) := \frac{\sum_{i=1}^{n}c_{max}^i}{\sum_{i=1}^{n}c_{min}^i},
\end{equation}
where $\eta_{eff}(\pi)$ is defined to be $1$ in case $\sum_{i=1}^{n}c_{min}^i=0$, as this implies $c(\pi)=0$ (due to Assumption~\ref{assum:bounds}), meaning that there is no uncertainty.
\noindent The following result is an implication of the above formulas.
\begin{theorem}\label{thm:bound}
	(\emph{Bound}).
	Given a problem $\mathcal{P}$ as in Def.~\ref{def:new_problem}, an optimal plan $\pi$ w.r.t. any $\Phi \subseteq \Theta_\Sigma$, satisfies the following relation for its cost:
	\begin{equation}\label{eq:bound}
		c(\pi) \leq c^* \times \eta_{eff}(\pi).
	\end{equation} 
\end{theorem}
\noindent \textit{Proof.} 
Denote by~$c_{min}^i, c_{max}^i$ the lower and upper bounds, for the cost of the $i$th action in~$\pi$.
By definition of optimality of~$\pi$ w.r.t. $\Phi$, it follows that
\begin{equation*}
	\sum_{i=1}^{n}c_{min}^i \leq c^*.
\end{equation*}
From the right inequality of~\eqref{eq:plan_bounds} and Equation~\eqref{eq:eta_eff}, we get
\begin{equation*}
	c(\pi) \leq \sum_{i=1}^{n}c_{max}^i = \sum_{i=1}^{n}c_{min}^i \times \eta_{eff}.
\end{equation*}
Substituting the former inequality into the latter yields Inequality~\eqref{eq:bound}.\hfill$\square$
\begin{corollary}
	Thm.~\ref{thm:bound} can be used to assure that an optimal plan $\pi$ w.r.t. any $\Phi \subseteq \Theta_\Sigma$ with $\eta_{eff}(\pi) \le \epsilon$ is indeed $\epsilon$-optimal.
	This also implies that algorithms can try to reduce planning time by choosing $\Phi$ that achieves this efficiently.
\end{corollary}

\subsection{Base Algorithm}\label{subsec:base_alg}

Consider two approaches for solving $\mathcal{P}$ from Def.~\ref{def:new_problem} based on cost-optimal forward states-space search:
\begin{enumerate}
	\item Perform standard planning using all the estimators available per edge whenever $c(e)$ is needed. We call this \emph{estimation-indifferent} planning. 
	\item Employ a \emph{fully lazy} approach w.r.t. edge cost estimation, by initially applying just one estimator per edge (whenever $c(e)$ is needed), until a plan is found. 
	Then, in case the requirement on the overall uncertainty of the plan is not met, improve the estimation of the first edge possible along the path implied by the plan, and start search from the beginning. Re-iterate until an $\epsilon$-bounded optimal solution is found or all estimation options are exhausted. 
\end{enumerate}
The first approach simply ignores the time needed for cost estimation, and thus can be expected to perform poorly in cases where considerable runtime could be saved on estimation, while the second approach can be seen to be wasteful in cases where significant time spent on search could have been saved.
Clearly, the approaches presented above are two extremes, hence it seems favorable to try to find a compromise between time spent on search and time spent on estimation.

\asec (\emph{$A^*$ with Synchronous Estimations of Costs}, Alg.~\ref{alg:ASEC}) is a modification of $A^*$, which balances search and estimation times.
It uses the required $\epsilon$-bound \emph{during the search}, as a threshold for allowed cost uncertainty. This lets it decide when to apply cost estimation, by comparing the $\epsilon$-bound to the accumulated~$\eta_{eff}$ along the current path examined.

\begin{algorithm}[htb]
	\caption{$A^*$ with Synchronous Estimations of Costs}
	\label{alg:ASEC}
	\textbf{Input}: Problem $\mathcal{P}=(\Sigma,\Theta_\Sigma,s_0,S_g)$, target~$\epsilon$ \\
	\textbf{Parameter}: Procedure $\mathsf{GetEstimator}(\cdot)$\\
	\textbf{Output}: Plan $\pi$, bound $\eta_{eff}$
	\begin{algorithmic}[1] 
		\STATE $g_{min}(s_0) \gets 0$; $g_{max}(s_0) \gets 0$ 
		\STATE OPEN $\gets \emptyset$; CLOSED $\gets \emptyset$
		\STATE Insert $s_0$ into OPEN with $f(s_0)=h(s_0)$
		\WHILE {OPEN $\neq \emptyset$}
		\STATE $n \gets$ best node from OPEN
		\IF {$Goal(n)$}
		\RETURN {$trace(n), g_{max}(n) / g_{min}(n)$}
		\ENDIF
		\STATE Insert $n$ into CLOSED 
		\FOR {\textbf{each} successor $s$ of $n$}
		\IF {$s$ not in OPEN $\cup$ CLOSED}
		\STATE $g_{min}(s) \gets \infty$
		\ENDIF
		\STATE $\eta_{eff} \gets \infty$; $\underline{g} \gets 0$
		\STATE $\theta \gets \mathsf{GetEstimator((n,s))}$
		\WHILE {$\eta_{eff} > \epsilon~\AND~\underline{g}< g_{min}(s)~\AND~\theta \neq \emptyset$}
		\STATE $\underline{c}, \bar{c} \gets$ apply$(\theta)$
		\STATE $\underline{g} \gets g_{min}(n) + \underline{c}$; $\bar{g} \gets g_{max}(n) + \bar{c}$
		\STATE $\eta_{eff} \gets \bar{g} / \underline{g}$
		\STATE $\theta \gets \mathsf{GetEstimator((n,s))}$
		\ENDWHILE
		\IF {$\underline{g} < g_{min}(s)$}
		\STATE $g_{min}(s) \gets \underline{g}$; $g_{max}(s) \gets \bar{g}$  
		\IF {$s$ in OPEN $\cup$ CLOSED}
		\STATE Remove $s$ from OPEN and CLOSED
		\ENDIF
		\STATE Insert $s$ into OPEN with $f(s)=g_{min}(s)+h(s)$
		\ENDIF
		\ENDFOR
		\ENDWHILE
		\RETURN {$\emptyset, \infty$}
	\end{algorithmic}
\end{algorithm}

\asec replaces the accumulated cost $g(\cdot)$ of $A^*$ with accumulated bounds $g_{min}(\cdot)$ and $g_{max}(\cdot)$.
$f$-values are computed based on lower bounds of costs, i.e., $f(s)=g_{min}(s)+h(s)$. 
Lines 10--20 replace the $A^*$ conditional step of adding $c(n,s)$ to $g(s)$, with a
 loop over possible cost estimators for the current edge, until the bound $\epsilon$ is met, a cheaper path exists, or no estimators are left.
$\mathsf{GetEstimator}(\cdot)$ is a procedure that receives an edge $e$ and returns an estimator for it (from the set $\Theta(e)$) that has not yet been applied.

\asec's theoretical guarantees (below) assume nothing about the ordering of the estimators selected by $\mathsf{GetEstimator}(\cdot)$. However,
to save runtime, one can impose an order on the set of estimators $\Theta((n,s))$, by increasing $\tau_i$. Then, each time $\mathsf{GetEstimator(\cdot)}$ is invoked on the edge $e=(n,s)$, it returns the estimator $\theta_i^e$ with the shortest runtime.  This is the approach we took in the experiments below.
Note that regardless of the estimators' order induced by $\mathsf{GetEstimator}(\cdot)$, the bounds for each edge can only be tightened with additional estimates (which is obtained from Assumption~\ref{assum:bounds} and by keeping the tightest bounds found).
Lastly, computation of the heuristic $h(\cdot)$ is based on lower bounds, obtained by using the (projected) cheapest estimations, i.e., for every relevant edge $e$ it uses the lower bound estimate given by the first estimator returned by $\mathsf{GetEstimator}(e)$.

\subsubsection{Analysis of \asec.}
We first clarify the meanings of completeness and soundness w.r.t. $\mathcal{P}$ from Def.~\ref{def:new_problem}.

\noindent \textbf{$\epsilon$-Completeness} An algorithm is said to be $\epsilon$-\emph{complete} if it guarantees to find an $\epsilon$-optimal plan, in case there exists a plan that can be verified to be $\epsilon$-optimal using the given set of estimators $\Theta_\Sigma$.

\noindent \textbf{$\epsilon$-Soundness} An algorithm is said to be $\epsilon$-\emph{sound} if every time it returns a plan reported as $\epsilon$-optimal, the plan returned is indeed $\epsilon$-optimal.

\asec employs a best-effort approach. It returns the first plan it finds, even if it is not $\epsilon$-optimal. However, when it returns an $\eta_{eff}$ that meets the bound, then the plan is guaranteed to satisfy $\epsilon$-optimality.

\begin{theorem}[$\epsilon$-Soundness]\label{thm:asec_sound}
	Provided with a consistent heuristic $h(\cdot)$, $\mathsf{ASEC}$ is $\epsilon$-sound.
\end{theorem}
\begin{proof}[Proof sketch]
	By induction on the order of nodes entering OPEN, starting from the base case of the start node, we show that each node $n$ in OPEN satisfies $\eta_{eff}(n) := g_{max}(n) / g_{min}(n)$ (where the case of $g_{min}(n)=0$ is considered $\eta_{eff}(n)=1$, as there is no uncertainty).
	Thus, if a plan $\pi$ is found, terminating at $s_g$, it necessarily means that $\eta_{eff}(s_g)=\eta_{eff}(\pi)$ is returned (i.e., the $\eta_{eff}$ returned is correct).
	In addition, using the same analysis of $A^*$ and a consistent $h(\cdot)$, $\pi$ is guaranteed to be optimal w.r.t. some $\Phi \subseteq \Theta_\Sigma$.
	Hence, if $\eta_{eff}(\pi) \leq \epsilon$ is satisfied, then relying on Thm.~\ref{thm:bound}, $\pi$ is an $\epsilon$-optimal plan, whereas in case it is not satisfied, then $\pi$ is not guaranteed by \asec to be $\epsilon$-optimal.
\end{proof}

However, not every plan returned will meet the target bound. It might even be that in generating the resulting plan, not all estimation options are exhausted, thus \asec could return a plan which does not satisfy $\epsilon$-optimality, while there exists a plan which does. Therefore it is \textit{not $\epsilon$-complete in the general case}. However, under some conditions, it is $\epsilon$-complete (Thm.~\ref{thm:asec_guar}).
Intuitively, if every edge can be estimated to the desired degree of certainty, then starting from a bound on the path cost which is under the threshold, and adding only edges with ``good enough'' cost estimates, it is possible to retain all paths explored with accumulated $\eta_{eff}$ lower or equal to $\epsilon$.
Then, the first solution found (if it exists) necessarily meets the requirement.

%
%
%

\begin{theorem}[Special $\epsilon$-Completeness]\label{thm:asec_guar}
	Given a problem $\mathcal{P}=(\Sigma,\Theta_\Sigma,s_0,S_g)$, with sub-optimality bound~$\epsilon$, if every edge $e \in \mathcal{E}$ can be estimated using $\Theta(e)$ such that $c_{max}(e) / c_{min}(e) \leq \epsilon$ (or $c_{min}(e)=c(e)=0$) holds, and a consistent heuristic $h(\cdot)$ is used, then $\mathsf{ASEC}$ is $\epsilon$-complete.
\end{theorem}

\begin{proof}[Proof sketch]
First, the fact that \asec terminates in finite time follows from the same analysis of $A^*$ and since there is a limited number of cost estimations per edge, where each of them has finite runtime.
Second, \asec is trivially complete in the regular sense, i.e., if a plan exists, then \asec necessarily returns such one (this again follows from the same analysis of $A^*$).
Third, similar to the proof of Thm.~\ref{thm:asec_sound}, and using the fact that for every edge $e$ the bound $c_{max}(e) / c_{min}(e) \leq \epsilon$ (or the equality $c_{min}(e)=c(e)=0$) can be achieved, it can be shown by induction that each node $n$ in OPEN satisfies $\eta_{eff}(n) = g_{max}(n) / g_{min}(n) \leq \epsilon$ (where again, $g_{min}(n)=0$ is considered as $\eta_{eff}(n)=1$).
Therefore, if an $\epsilon$-optimal plan exists, then \asec will necessarily return such one.
\end{proof}

\subsection{Post-Search Extension}\label{subsec:extensions}
We provide a simple post-search procedure (Alg.~\ref{alg:ESE}) that aims to reduce the uncertainty expressed by $\eta_{eff}$ in cases where the algorithm terminates with $\eta_{eff}(\pi) > \epsilon$. 
When \asec terminates, there may still be unused estimators for edges along the path that corresponds to the plan found.  Thus if $\eta_{eff}(\pi)$ doesn't meet the $\epsilon$ requirement, then we may try to reduce it until the condition is met, without continuing the search.

The basic idea is to loop over all the edges along the path corresponding to the plan found, and for each one try to reduce its uncertainty by applying unused estimators.
We note that the tightest lower bound we may use to calculate $\eta_{eff}(\pi)$, cannot be greater than the lowest $f$-value of nodes in OPEN, as it should reflect a correct lower bound to the goal, and it is possible that there exists a different plan $\tilde{\pi}$ with $c_{min}(\tilde{\pi})$ equal to that $f$-value. This is the purpose of comparing the newest lower bound $c_{min}(\pi)$ to $g_{min}(alt)$ and taking the minimum between them.

\begin{algorithm}[htb]
	\caption{End of Search Estimations (\ese)}
	\label{alg:ESE}
	\textbf{Input}: \asec's inputs and variables before termination\\
	\textbf{Parameter}: Procedure $\mathsf{GetEstimator}(\cdot)$\\
	\textbf{Output}: Bound $\eta_{eff}$
	\begin{algorithmic}[1] 
		\STATE $alt \gets$ best node from OPEN
		\FOR {\textbf{each} edge $e$ that corresponds to an action from $\pi$}
		\STATE $\theta \gets \mathsf{GetEstimator(e)}$
		\WHILE {$\eta_{eff} > \epsilon~\AND~\theta \neq \emptyset$}
		\STATE $\underline{c}, \bar{c} \gets$ apply$(\theta)$
		\STATE update $c_{min}(\pi), c_{max}(\pi)$ using $\underline{c}, \bar{c}$
		\IF {$g_{min}(alt) < c_{min}(\pi)$}
		\STATE $\eta_{eff} \gets c_{max}(\pi) / g_{min}(alt)$
		\ELSE
		\STATE $\eta_{eff} \gets c_{max}(\pi) / c_{min}(\pi)$
		\ENDIF
		\STATE $\theta \gets \mathsf{GetEstimator(e)}$
		\ENDWHILE
		\IF {$\eta_{eff} \leq \epsilon$}
		\RETURN {$\eta_{eff}$} 
		\ENDIF
		\ENDFOR
		\RETURN {$\eta_{eff}$}
	\end{algorithmic}
\end{algorithm}

\section{Empirical Evaluation}\label{sec:emp}
The theoretical properties of \asec give little insight as to its runtime behavior and expected success rate. 
We therefore conduct extensive experiments to evaluate \asec in a variety of benchmark planning problems where we control the need for estimation and target sub-optimality bound.
Our experiments are carried out using \emph{PlanDEM} (Planning with Dynamically Estimated Action Models), a planner that provides a concrete implementation for \asec and \ese, 
in C++. PlanDEM modifies and extends Fast Downward (FD)~\cite{helmert-jair2006} (v20.06), and inherits its capabilities, including the $h_{max}$ heuristic~\cite{bonet2001planning} 
used in the experiments. 
The underlying FD search algorithm and data structures were modified appropriately. 
All experiments were run on an Intel i7-6600U CPU (2.6GHz), with 16GB of RAM, in Linux.

We generated planning problems by using standard problems from previous planning competitions, modified to include synthetic estimators at several levels of uncertainty and computational cost. Specifically, we selected 20 domains \& problems with action costs~\cite{fox2003pddl2} (see the technical appendix).
For each, we generated a random variant with a target $\epsilon$, and varying the number of actions with estimated (rather than precise) costs with a probability $p_1$.
When we set $p_1=1$, the costs of all ground actions (edges in $\mathcal{G}_\Sigma$) are estimated by the planner. When $p_1=0$, all costs are given precisely. When $0<p_1<1$, only some ground actions require estimation of the cost. 
Thus a pair $\epsilon, p_1$ creates a new instance of a benchmark problem.

Each estimated-cost edge $e$ with precise cost $c(e)$ is provided with a set of three estimators of $c(e)$, 
$\Theta(e)=\{\theta_1^e, \theta_2^e, \theta_3^e\}$. 
These have $c_{max}(\theta_i^e)/c_{min}(\theta_i^e)$ uncertainty ratios of 4, 2, and 1 (the correct value), resp. 


Since the problems we address have a target bound $\epsilon$, but no other solution ranking criteria, we choose to prefer time over smaller uncertainty as a secondary ranking criterion.
Given two $\epsilon$-optimal solutions $\pi_1, \pi_2$ with corresponding $\eta_{eff}^1, \eta_{eff}^2$ and runtimes $t^1, t^2$, that satisfy $\eta_{eff}^1 < \eta_{eff}^2$ and $t^1 > t^2$, we consider $\pi_2$ to be better.

\subsection{Evaluation of \asec.}\label{sub:ASEC}
We take $p_1$ to be one of $\{0.01, 0.05, 0.1, 0.25, 0.5, 0.75, 1\}$ and $\epsilon$ values in the range $[1,4]$ in jumps of 0.25, resulting in a total of 1820 runs.
Note that $\epsilon=1$ requires precise cost, while $\epsilon=4$ is equivalent in our setting to having no requirement at all on precision (as the highest uncertainty ratio given by an estimator is 4). As one of the estimators has uncertainty ratio of 1, the cost of each edge can be eventually estimated perfectly, and thus Thm.~\ref{thm:asec_guar} holds (i.e., the achievable $\eta_{eff}$ is 1).  Therefore,  \asec will always terminate with an $\epsilon$-optimal solution.

\subsubsection{Comparison to Baseline.}\label{subsub:alternatives}
As the planning problem defined in this paper is novel, there is no baseline for comparison other than the relatively simple approaches suggested earlier: \emph{estimation-indifferent} planning utilizing all available estimators, and \emph{fully lazy} planning re-iterating the search until an $\epsilon$-optimal solution is found.  This latter approach proved to be several orders of magnitude slower than the other two on the problems tested. The only scenario where it might be a contender is in the case of  very small scale problems, with extremely long average cost estimation runtime. This may be relevant for some applications, but it seems too limiting for the general case.  

We therefore contrast PlanDEM with estimation-indifferent planning.
We begin by examining \asec's runtime behavior from the perspective of estimators utilization. In particular, as every edge considered has to be estimated at least once, we focus on
the number of expensive estimators, i.e., (those with tighter ratios 2, 1).

Fig.~\ref{plot:estimations} plots the ratio between the actual number of expensive estimations that \asec used and the maximum number of potential expensive estimations on the vertical axis, against the target $\epsilon$ bound. Thus a lower ratio indicates an improved result, i.e., less expensive estimations used. 
The estimation-indifferent planning approach always uses all estimators, and therefore its ratio is always 1,  indicated by the
the straight solid horizontal black line at  the top of the figure. The figure shows different curves for the several $p_1$ values. Each point on each curve averages 20 planning runs. The curves of $p_1=\{0.01, 0.05\}$ were left out for clarity. 


\begin{figure}[htb]
	\centering
	\includegraphics[width=0.8\columnwidth]{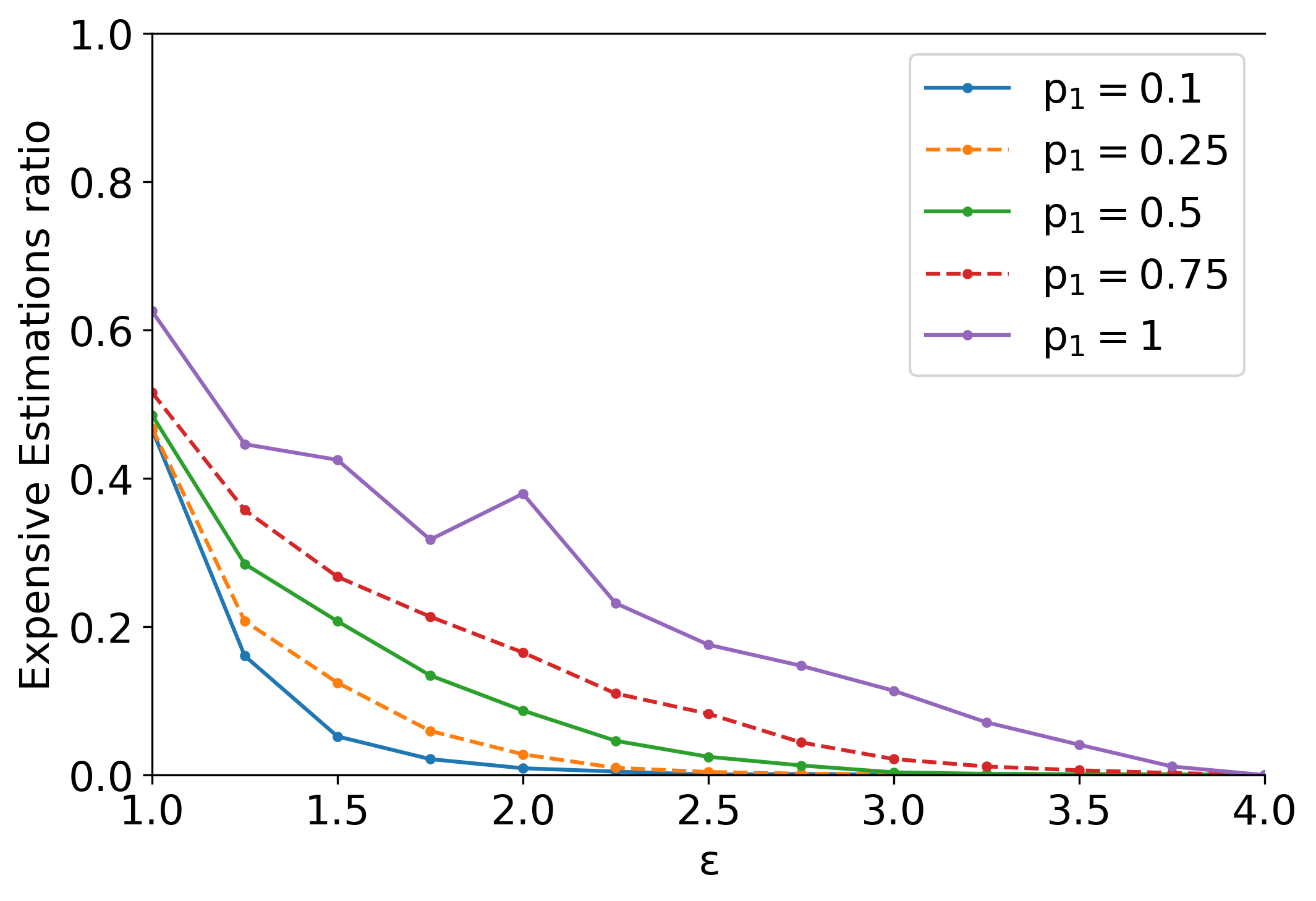} 
	\caption{Normalized expensive estimations versus target bound $\epsilon$. 
		 The curves from the bottom (blue, solid), 
		 to the top (purple, solid) correspond to the $p_1$ values $0.1, 0.25, 0.5, 0.75$ and $1$. Success rate is 100\% for all runs.}
	\label{plot:estimations}
\end{figure}

We first compare the estimation-indifferent approach with \asec. Consider the worst-case where the cost of \emph{all} actions is estimated ($p_1=1$, solid purple), and the target bound requires maximum accuracy ($\epsilon=1$). Here, \asec utilizes only 62\% of the expensive estimators. For lower $p_1$, it improves up to 46\% ($p_1=0.1$). This is due to the condition on $\underline{g}$ in line 14 of \asec, which avoids unnecessary estimations in case another path is examined that leads to a node that is already in OPEN and has lower $f$-value.  Without sacrificing accuracy, \asec is superior to its basic competitor that uses full estimation (indicated as the straight line of ratio 1).

Next, the generally decreasing trend of the curves in Fig.~\ref{plot:estimations} 
suggests that substantial savings may be achieved even by minor relaxation of the uncertainty bound (allowing greater $\epsilon$).
Furthermore, as $p_1$ decreases the curves drop rapidly, allowing \asec to further increase its efficiency.

\subsubsection{Runtime and Quality}\label{subsub:runtime_quality}
Figure~\ref{plot:runtime} shows actual and approximate planning runtime (measured in CPU seconds), on the vertical axis (logarithmic scale). As the time taken by expensive estimators is
arbitrary in these experiments, we wanted to get a feel for the planning runtime under various assumptions of computation time. The bottom curve in the figure is the actual runtime of
PlanDEM without the estimation. The other curves, bottom-to-top, show what the runtime would be had we added the estimation runtime (number of expensive estimates multiplied
by assumed per-estimate runtime).  This emphasizes the importance of avoiding unnecessary estimations, which dominate the total runtime when estimation time is high (compare to the heuristic estimation time, contained within the baseline).  It also strengthens the argument mentioned above that significant potential savings can be obtained by even slightly loosening the uncertainty bound.

\begin{figure}[htb]
	\centering
	\includegraphics[width=0.8\columnwidth]{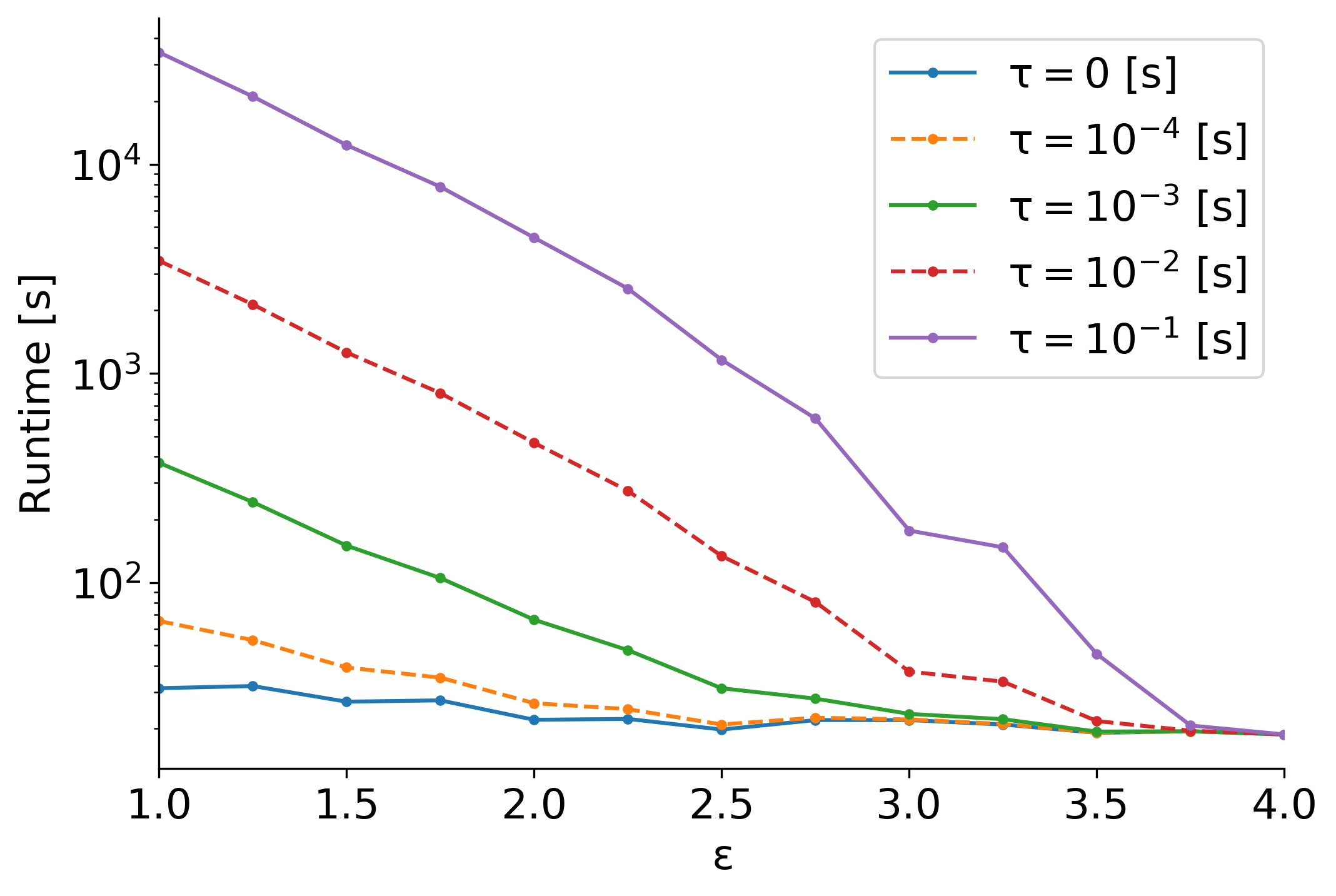} 
	\caption{PlanDEM runtime in CPU seconds. Bottom (blue, solid) 
		to top (purple, solid) curves correspond to per-estimation times of $0, 10^{-4}, 10^{-3}, 10^{-2}$ and $10^{-1}$ seconds for the expensive estimators. $p_1=0.5$ for all cases.}
	\label{plot:runtime}
\end{figure}

We now move to examine the quality of the solutions produced by \asec.  Figure~\ref{plot:eta} shows mean $\eta_{eff}$ values achieved for different target $\epsilon$ bounds. 
Recall that we consider shorter planning time to be better than achieving higher accuracy.
Hence, the fact that $\eta_{eff}$ tends to $\epsilon$ when $p_1$ is high indicates that \asec is efficient w.r.t. usage of expensive estimations.
When $p_1$ values are low, most costs are precisely known, and overall uncertainty should be low even with a small amount of expensive estimations (see Fig.~\ref{plot:estimations}). 
Indeed, $\eta_{eff}$ values are well below the target bounds. 

\begin{figure}[htb]
	\centering
	\includegraphics[width=0.8\columnwidth]{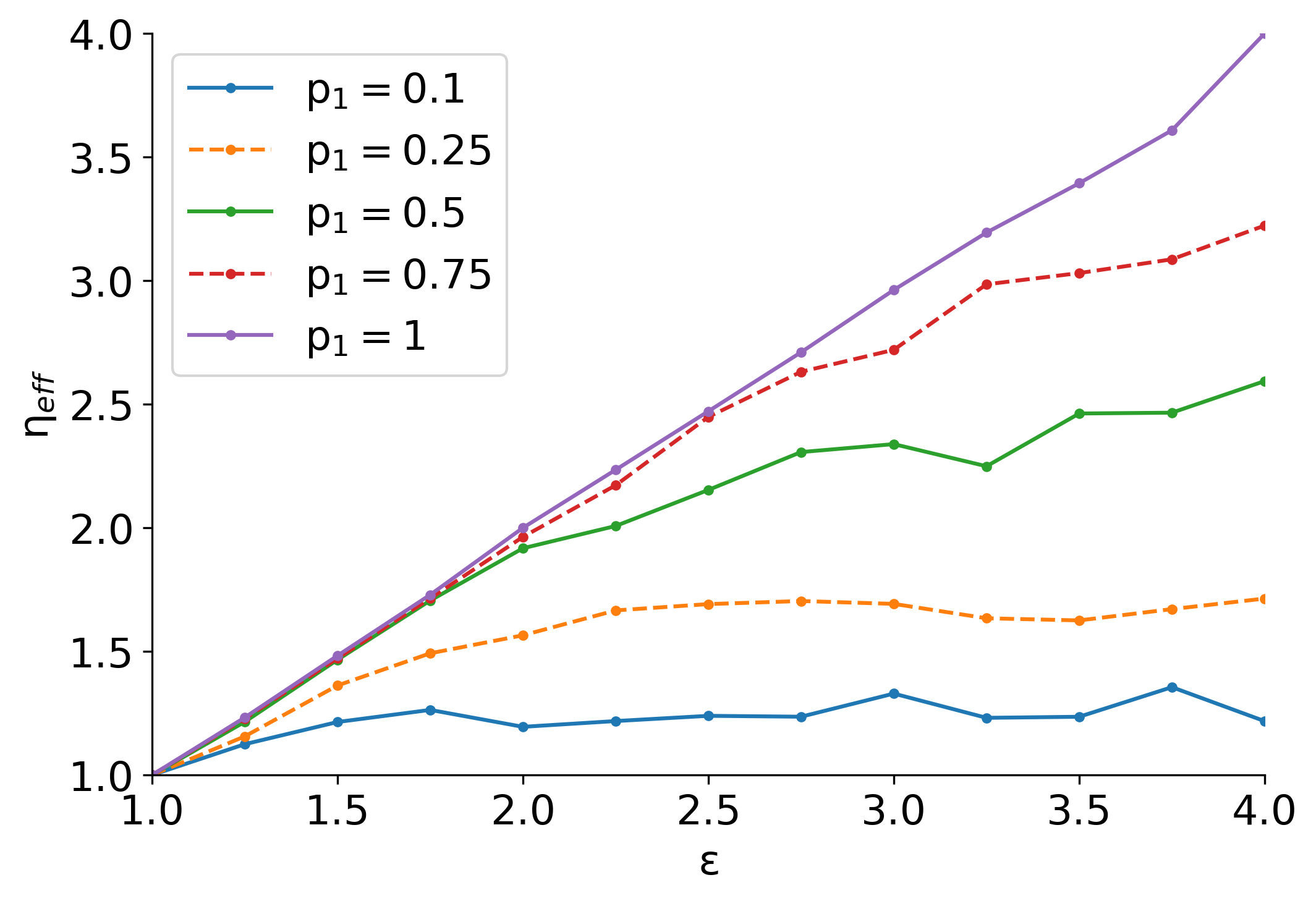} 
	\caption{Mean $\eta_{eff}$ vs $\epsilon$ target bounds. The curves bottom (blue, solid) 
		to top (purple, solid) correspond to the $p_1$ values $0.1, 0.25, 0.5, 0.75$ and $1$.}
	\label{plot:eta}
\end{figure}

\subsubsection{Diminishing Marginal Improvement}\label{subsub:DMI}
The characteristic shape of the curves in Fig.~\ref{plot:estimations}, is intriguing. We analyze the accumulated $\eta_{eff}$ that \asec uses as threshold for estimation improvement.
First, we introduce notation that separates an edge $e^j$ along a path examined by \asec.
\begin{equation}\label{eq:DMR_notation}
	\eta_{eff} = \frac{\sum_{i\neq j}c_{max}^i + c_{max}^j}{\sum_{i\neq j}c_{min}^i + c_{min}^j} := \frac{N + \alpha}{D + \beta}.
\end{equation}
Second, we examine the effect of estimation improvement for the cost $c(e^j)$.
We note that improvement means that the upper bound $\alpha$ decreases and/or the lower bound $\beta$ increases, resulting in an overall decrease of $\eta_{eff}$.
Taking partial derivatives with respect to the bounds gives:
\begin{equation}\label{eq:partial_derivatives}
	\frac{\partial \eta_{eff}}{\partial \alpha} = \frac{1}{D + \beta}~,~\frac{\partial \eta_{eff}}{\partial \beta} = \frac{-(N + \alpha)}{(D + \beta)^2}.
\end{equation}
It can be seen that increasing $\beta$ makes both derivatives smaller in absolute value, namely it reduces the effect that shrinking the bounds of $c(e^j)$ has on $\eta_{eff}$.
Similarly, decreasing $\alpha$ makes $|\frac{\partial \eta_{eff}}{\partial \beta}|$ smaller, hence the effect that improving $\beta$ has on $\eta_{eff}$ is diminished.
On the other hand, changing $\alpha$ does not affect $\frac{\partial \eta_{eff}}{\partial \alpha}$, implying that two consecutive identical improvements in $\alpha$ have the same affect on $\eta_{eff}$ (i.e., a fixed $\Delta \alpha$ causes a fixed $\Delta \eta_{eff}$). 
However, notice that every additional improvement of the uncertainty in $c(e^j)$ by a factor of $\delta > 1$, namely $\eta_k(c(e^j))/\eta_{k+1}(c(e^j)) = \delta$, means a smaller absolute change in the bounds, as
\begin{equation}\label{eq:def_alpha}
	\alpha_1 = \alpha_2 \times \delta = \alpha_3 \times \delta^2,
\end{equation}
translates to
\begin{equation}\label{eq:ineq_alpha}
	\alpha_1 - \alpha_2 = \alpha_3 \times \delta \times (\delta - 1) > \alpha_3 \times (\delta - 1) = \alpha_2 - \alpha_3.
\end{equation}
Equations~\eqref{eq:def_alpha} and \eqref{eq:ineq_alpha} imply that every additional improvement of the upper bound $\alpha$ by an identical factor has a smaller effect on the accumulated $\eta_{eff}$.

To sum up, there is a diminishing marginal return effect in the improvement of $\eta_{eff}$ due to better estimation of an action's cost.
The implication is that requiring a looser target $\epsilon$ indirectly enables \asec to utilize less expensive estimations, according to a \emph{convex curve}.
Changing $\epsilon$ also affects the order of node expansion, which in turn affects the number of estimations, hence this analysis is approximate.
Nevertheless, it seems to explain the curves in Fig.~\ref{plot:estimations}.

%
%

\subsection{Evaluation of \ese.}\label{sub:ESE}
The section above examines \asec when it can always succeed in finding a solution, and thus \ese is never invoked. This section examines cases where it might not.
We introduce probabilities for the existence of the second and third estimators, denoted by $p_2$ and $p_3$. 
When setting either $p_2$ or $p_3$ to values different than 1, the achievable $\eta_{eff}$ increases, as tighter estimates are no longer available.
This creates cases where \asec can fail, and \ese might be applicable, and will be invoked.
For our experiment we used $p_1, p_2 \in \{0.25, 0.5, 0.75, 1\}$, $p_3 \in \{0.25, 0.5, 0.75\}$ and $\epsilon \in \{1.5, 2, 2.5 ,3 ,3.5\}$. This resulted in 4800 
planning problem variants to be solved. 

The results are summarized in Table~\ref{table:ESE}. Each row corresponds to a different $\epsilon$ target, and summarizes 960 problems. 
Left to right, the columns show the number of problems (in parentheses, percentage of problems) in which \asec succeeded, \ese was applicable and invoked, \ese was successful, the resulting
$\eta_{eff}$ and usage of costly estimations for both \asec and \ese.  The table points to several insights.


\begin{table*}[htb]
	\centering
	\begin{tabular}{r r r r r r} 
		$\epsilon$ & \asec Success (\%) & \ese Invoked (\%) & \ese Success (\%) & $\eta_{eff}$ \asec (\ese) & Costly Estimations \asec (\ese)\\
		\hline
		1.5 & 276 (28.75) & 174 (18.13) & 50 (28.74) & 2.23 (1.98) & 132566045 (511) \\ 
		2   & 548 (57.08) & 171 (17.81) & 59 (34.50) & 2.63 (2.14) & 145826527 (602) \\
		2.5 & 701 (73.02) & 145 (15.10) & 52 (35.86) & 2.95 (2.33) & 136260264 (544) \\
		3   & 806 (83.96) & 97 (10.10)& 55 (56.70) & 3.29 (2.51) & 93636149 (382) \\
		3.5 & 897 (93.44) & 38 ~~(3.96) & 23 (60.53) & 3.67 (2.68) & 29996783 (144) \\
	\end{tabular}
	\caption{Summarized data of \asec \& \ese experiments (960 runs per $\epsilon$ value). The values in column 5 are averages of $\eta_{eff}$ over all instances where \ese took place, and the values in column 6 are sums of costly estimations over all those instances.}
	\label{table:ESE}
\end{table*}

First, the number of \ese invocations is not equal to the number of \asec failures. There are cases where \asec exhausted all estimation options for the actions in the plan
and thus \ese is not applicable.  As the success ratio of \asec increases with $\epsilon$, the number of \ese invocations decreases (as this procedure is required less frequently).

Second, a key insight is that \ese has a considerable success ratio, which increases with $\epsilon$.  This is striking,
given that it utilizes only a small amount of estimates w.r.t. the amount \asec uses, i.e., its affect on the total runtime is minor.

To explain this result, we examine the change in $\eta_{eff}$ due to \ese.
The data in column 5 from the table indicates values averaged over both successful and failed runs (when the bound is not met), without normalization. 
We therefore examine the normalized relative change in $\eta_{eff}$, defined as
\begin{equation*}
	\eta_{rel}:=(\eta_{eff}^{\mathsf{ESE}}-\eta_{eff}^{\mathsf{ASEC}})/(\eta_{eff}^{\mathsf{ASEC}}-1),
\end{equation*}
where the value 1 in the denominator, which is the best (lowest) achievable $\eta_{eff}$, serves as to normalize the change.

The mean normalized relative change is only  $-5.72\%$. 
This implies that \asec typically comes very close to the target (up to roughly several percent). It then becomes relatively easy for \ese to use a few more estimates, to drive $\eta_{eff}$ below the target.



\section{Related Work and Discussion}
\asec allows the planner to decide between multiple estimators for action costs. This is reminiscent of
work on multi-fidelity behaviors~\cite{winner00}, where a behavior-based robot selects between different versions of the same
behavior (controller) depending on uncertainty and value of information. In addition to our focus on planning rather than reactive behavior selection, we also provide a formal framework, with guarantees, for the selection process.  

\asec's repeated evaluation and selection between cost estimators is also related to the  problem of choosing between multiple heuristic functions during the planning process~\cite{karpas2018rational}.  A key difference is that heuristics are applied to states (vertices in the state-space graph), while costs apply to actions (edges). Therefore, the two approaches are complementary. 

In general graph search, cost uncertainty is considered by assuming explicit or implicit knowledge about edge cost distributions, and then solved by performing calculations involving the full graph, typically minimizing expectation~\cite{kwon2013robust, shahabi2015robust}. \asec makes no such assumptions.

Two related efforts tackle graph search in the context of motion planning, where edge evaluation (specifically, edge existence verification) is expensive~\cite{narayanan2017heuristic, mandalika2019generalized}. 
 Both efforts suggest using coordination mechanisms that decide when to search and when to evaluate edges, so as to minimize overall planning time. This approach is specific to motion planning, and does not lend itself to general planning.
In particular, they build a full graph description of the state space (corresponding to the problem state space), an intractable process in general planning problems.  Additionally, they consider edge evaluation as a binary function (an edge either exists or not), whereas \asec assumes multiple numeric estimations, thus an edge may be evaluated more than once. 

As mentioned earlier, \asec's theoretical properties hold without any assumptions on the order of the estimators given by $\mathsf{GetEstimator}(\cdot)$. 
Nevertheless, \asec is designed to take advantage of estimators organized according to an increasing computational burden.
\emph{This means that even partial information on the expected runtime of estimators---which is oftentimes a realistic assumption---can be utilized, as the more $\mathsf{GetEstimator}(\cdot)$ adheres to the desired order, the better the results of \asec (i.e., it is becomes more efficient).}

An underlying assumption of this paper is that cost estimation bounds can be obtained.
Clearly, in many cases bounds that are both strict and meaningful may not be available, e.g., when using statistical estimators (such as neural networks).
Nevertheless, various statistical measures of the estimations (notably measured standard deviations) can be used instead as bounds, in which case the plan accuracy bound should be considered soft (or approximate) rather than strict.

The ideas proposed in this paper are part of a broader vision to improve real-world planning, by removing barriers between problem modeling and planning.
Allowing the planner to choose and invoke external computational modules during the planning process in order to acquire numerical attributes that characterize the problem (such as action costs), while retaining the problem structure obtained by declarative and symbolic approaches, provides the best of both worlds:
\begin{itemize}
	\item Any kind of computational procedure can be used to estimate numeric model parameters. In particular, there are no restrictions on the type of data nor on mathematical operations, including black-box estimators such as neural networks. 
	\item The uncertainty level of the problem description can be controlled (limited only by the best achievable estimators) to trade-off accuracy vs. computation time. 
	\item Domain independent planners can still be applied.
\end{itemize}
Consequently, considerable improvement in problem modeling can be achieved, along with the generality and performance given by state-of-the-art domain independent planning technology.

\section{Conclusions}
This paper introduced a generalization of deterministic planning to include action cost uncertainty, quantified through cost estimators, and applied dynamically during the planning phase.  As estimation processes are often expensive in terms of runtime, it is preferable to let the planner use them on demand, to save unnecessary computations. We introduced  \asec---a generalization of $A^*$---and showed that it is sound but incomplete (in general), though it is complete under some conditions. We then introduced a post-search procedure that increases \asec's success ratio by utilizing unused estimators. Extensive experiments show the algorithms select estimators so as to reduce estimation runtime.

However, much remains to be done.  Figure~\ref{plot:runtime} shows how the runtime of the planner would change, depending on the computation time of the estimates. The figure shows that even when the requirement for external estimation occurs in only half the ground actions, the planning runtime can be
completely dominated by the estimation process, when the target bound $\epsilon$ on the cost uncertainty tends to be tight (lower). Thus, an important direction for future research is to improve the selection of estimators so as to minimize their use. We are considering the use of 
meta-information, such as relying on predicated values for the bounds $c_{min}, c_{max}$ and/or $\tau$ prior to the actual application of an estimator. 

\section*{Acknowledgements}
The research was partially funded by ISF Grant \#2306/18 and BSF-NSF grant 2017764. Thanks to K. Ushi. 
E.W. is supported by the Adams Fellowship Program of the Israel Academy of Sciences
and Humanities and by Bar-Ilan University's President Scholarship.

\bibliography{short,weiss_references} 

\section{Technical Appendix}
We provide a few more implementation-related details to facilitate in reproducing the results reported in the paper, we supplement the proof sketches of Theorems 3 and 4, and we present an experiment with real data that further illustrates our approach with a concrete example.

\subsection{Domains and Problems Used in the Experiments}
For the empirical evaluation we used domains and problems that appeared in the international planning competitions of 2008, 2011, 2014 and 2018.
For each domain we chose two problems (ranging from small to large scale), and for each of them we tested all configurations reported in the paper (resulting in many different variants per original problem file).
The full list of domains (problems) follows below. The files can be retrieved conveniently from the online archive https://github.com/aibasel/downward-benchmarks.

\begin{itemize}
 \item  elevators-opt08-strips (p04, p06)
 \item  barman-opt11-strips (pfile01-003, pfile01-004) 
 \item floortile-opt11-strips (opt-p05-009, opt-p06-011)
 \item sokoban-opt11-strips (p04, p07)
 \item transport-opt11-strips (p02, p04)
 \item  woodworking-opt11-strips (p06, p12)
 \item tetris-opt14-strips (p03-4, p04-6)
 \item agricola-opt18-strips (p08, p10)
 \item  caldera-split-opt18-adl (p05, p10)
 \item data-network-opt18-strips (p17, p20)     

\end{itemize}        

\subsection{Costs and Estimators}
We generally used the original cost depicted in the PDDL domain and problem files, denoted as $c_{PDDL}(e)$.  However, PlanDEM supports solely positive integer action costs, as it extends FD and thus inherits its main software architecture and data structures.  As a result, for problems that have $c(e)=1$, it is not possible to bound an estimate from below. In these cases we changed the original cost so as to
be able to experiment with estimators that provide a bound lower than the original cost.

Thus, in order to obtain the desired uncertainty ratios for each edge $e$, we used  $c_{PDDL}(e)$ as a basis for the following transformation.
\begin{itemize}
    \item If the edge cost was to be estimated (i.e., with probability $p_1$ as described in the experiment section of the paper, then the list of estimators is 
        given by $\Theta(e)=\{\theta_1^e, \theta_2^e, \theta_3^e\}$, where the estimators respectively return
	\begin{align*}
		c_{min}^1=1 \times c_{PDDL}(e), c_{max}^1=4 \times c_{PDDL}(e),\\
		c_{min}^2=2 \times c_{PDDL}(e), c_{max}^2=4 \times c_{PDDL}(e),\\
		c_{min}^3=2 \times c_{PDDL}(e), c_{max}^3=2 \times c_{PDDL}(e). 
	\end{align*}
	In this case, the true cost is taken to be $c(e) = 2 \times c_{PDDL}(e)$.
\item Otherwise (true cost known and not estimated): $\Theta(e)=\{\theta_1^e\}$, where $\theta_1^e$ returns
	\begin{equation*}
		c_{max}^1=c_{min}^1=c_{PDDL}(e). 
	\end{equation*}
	 Hence, in this case $c(e) = c_{PDDL}(e)$.
\end{itemize}

\subsection{Supplements to Proofs}

\setcounter{table}{1}
\begin{table*}[htb]
	\centering
	\begin{tabular}{r r r r r r r r r r r} 
		$\epsilon$ & Success & $\eta_{eff}$ & $\theta_2$ Used & $\theta_2$ Saved & $\theta_2$ Usage (\%) & $T_0$[s] & $T_1$[s] & Saved $T_1$[s] & $T_{10}$[s] & Saved $T_{10}$[s] \\
		\hline
		1 & 0 & 1.18 & 2537169 &3858085 & 39.67 & 54.12 & 2591.29&3858.09 &25425.81 &38580.85  \\ 
		1.1   & 0 & 1.18 & 2537169 & 3858085 & 39.67 & 61.89 &2599.06 &3858.09 &25433.58 & 38580.85 \\
		1.2 & 1 & 1.18 & 2565641 & 4254910 & 37.62 & 66.16 &2631.80 &4254.91 & 25722.57&42549.10 \\
		1.3   & 1 & 1.28& 2653524 & 7091775 & 27.23 & 93.93 &2747.46 & 7091.78&26629.17 &70917.75 \\
		1.4 & 1 & 1.39 & 2469987 & 8222062 & 23.10  & 101.26&2571.25 &8222.06 &24801.13 &82220.62 \\
		1.5 & 1 & 1.47 & 2253078 & 8587468 & 20.78  & 102.10&2355.18 &8587.47 &22632.88 &85874.68 \\ 
		1.6   & 1 & 1.58 & 2047551 & 8909325 & 18.69  &106.90 &2154.45 &8909.33 &20582.41 & 89093.25\\
		1.7 & 1 & 1.66 & 1809278 & 8843618 & 16.98  &105.22 & 1914.50& 8843.62&18198.00 &88436.18 \\
		1.8   & 1 & 1.76& 1508847 & 9375424 & 13.86  &94.54 & 1603.39&9375.42 &15183.01 &93754.24 \\
		1.9 & 1 & 1.88 & 1149769 & 8486273 & 11.93  &72.78 &1222.55 & 8486.27&11570.47 &84862.73 \\
		2 & 1 & 1.96 & 722353 & 8102761 & 8.19  &84.37 &806.72 & 8102.76&7307.90 &81027.61 \\ 
		2.1   & 1 & 2.03 & 237430 & 6863791 & 3.34  &69.62 &307.05 &6863.79 & 2443.92&68637.91 \\
		2.2 & 1 & 2.05 & 733 & 6349392 & 0.01 & 60.23& 60.96&6349.39 &67.56 &63493.92 \\
		2.3   & 1 & 2.16& 31 & 6350040 & $5 \times 10^{-4}$ &61.55 &61.58 &6350.04 &61.86 &63500.40 \\
		2.4 & 1 & 2.24 & 7 & 6350479 & $10^{-4}$ &60.59 & 60.59&6350.48 &60.66 & 63504.79\\
	\end{tabular}
	\caption{Summarized data of \asec experiments with the Ontario Transportation Logistic problem. Notations: $\theta_2$ denotes the costly estimations, and $T_0, T_1, T_{10}$ correspond to planning times using $\tau_2=0,1,10$ ms, respectively. Usage and savings are w.r.t. to the results of estimation-indifferent planning.}
	\label{table:Ontatio}
\end{table*}

The proof sketches of Theorems 3 and 4 rely on the following argument.
\begin{lemma}[$\Phi$-Optimality]\label{lem:phi_optimality}
Provided with a consistent heuristic $h(\cdot)$, \asec necessarily returns an optimal plan w.r.t. some $\Phi \subseteq \Theta_\Sigma$.
\end{lemma}
In order to prove Lemma~\ref{lem:phi_optimality} we require another auxiliary result, which demonstrates that heuristic consistency is preserved when using lower bounds of edge costs for its calculation (instead of the true edge costs).
Note that often the term heuristic function, in the context of graph search (or AI planning), is an abuse of notation used refer to a more general computational procedure, that is parameterized by the edge costs of the graph (or by the action costs in the planning problem), so that only after fixing the costs, a standard heuristic is obtained. 
Further note that it is typical to attribute a theoretical property, such as consistency, to such a computational procedure, if the property holds for any heuristic obtained from the procedure after fixing the costs, regardless of their specific values.
For stating our result, we need to differentiate between the two. 
Hence, we call the more general computational procedure a parameterized heuristic and denote it as usual by $h(\cdot)$, and we denote a standard heuristic obtained by it using a subscript, e.g., $h_c(\cdot)$, where $c(\cdot)$ is a cost function defining the costs of the problem.


\begin{lemma}[Consistency Preservation]\label{lem:consistency}
Given a digraph $\mathcal{G} = (\mathcal{V}, \mathcal{E})$ and two edge cost functions $\alpha : \mathcal{E} \to [0,\infty),~\beta : \mathcal{E} \to [0,\infty)$ that satisfy $\alpha(e) \leq \beta(e)$ for every edge $e \in \mathcal{E}$, we obtain the weighted digraphs $\mathcal{G}_{\alpha} = (\mathcal{V}, \mathcal{E}, \alpha(\cdot)),~\mathcal{G}_{\beta} = (\mathcal{V}, \mathcal{E}, \beta(\cdot))$.
If $h(\cdot)$ is a consistent parameterized heuristic, then $h_{\alpha}(\cdot)$ is consistent w.r.t. $\mathcal{G}_{\beta}$.
\end{lemma}
\begin{proof}[Proof of Lemma~\ref{lem:consistency}]
Due to the fact that $h(\cdot)$ is a consistent parameterized heuristic, it immediately follows that $h_{\alpha}(\cdot)$ is consistent w.r.t. $\mathcal{G}_{\alpha}$. 
From the definition of consistency, this means that $h_{\alpha}(n) \leq \alpha((n,s)) + h_{\alpha}(s)$ and $h_{\alpha}(G)=0$ is satisfied for every node $n$ and every descendant $s$ of $n$, and every goal node $G$ in $\mathcal{G}_{\alpha}$.
Using $\alpha(e) \leq \beta(e)$ we obtain $h_{\alpha}(n) \leq \beta((n,s)) + h_{\alpha}(s)$, where again, this holds for every node $n$ and every descendant $s$ of $n$.
Additionally, every goal node $G$ in $\mathcal{G}_{\alpha}$ is also a goal node in $\mathcal{G}_{\beta}$, so $h_{\alpha}(G)=0$ is also satisfied for every goal node in $\mathcal{G}_{\beta}$.
Thus, $h_{\alpha}(\cdot)$ satisfies the definition of consistency w.r.t. $\mathcal{G}_{\beta}$.
\end{proof}
We can now prove Lemma~\ref{lem:phi_optimality}.
\begin{proof}[Proof of Lemma~\ref{lem:phi_optimality}]
As mentioned before, \asec uses heuristic values that are computed by $h(\cdot)$ (which is a consistent parameterized heuristic) and based on the lower bound estimate given for each relevant edge $e$ by the first estimator returned by $\mathsf{GetEstimator}(e)$. 
We obtain a (standard) consistent heuristic $h_{\alpha}(\cdot)$, with $\alpha(\cdot)$ defined to be an edge cost function, where the cost of each edge $e$ is taken to be the lower bound estimate of the first estimator returned by $\mathsf{GetEstimator}(e)$.
During the search \asec potentially utilizes additional estimations (as it tries to meet the bound $\epsilon$).
Denote by $\Phi \subseteq \Theta_\Sigma$ the set that includes exclusively all estimators invoked by \asec during the search, and further denote $\beta(\cdot)$ as an edge cost function, where the cost of each edge $e$ is taken to be the tightest (highest) lower bound estimate of all the estimators returned by $\mathsf{GetEstimator}(e)$ during the search.
Since $\alpha(e) \leq \beta(e)$ for every edge $e$, Lemma~\ref{lem:consistency} assures that $h_{\alpha}(\cdot)$ is consistent w.r.t. $\mathcal{G}_{\beta}$ (where its nodes and edges are defined by the digraph that corresponds to the planning problem $\mathcal{P}$).
Then, following the main proof arguments of $A^*$'s optimality, consistency of $h_{\alpha}(\cdot)$ implies non-decreasing $f$-values along any path, thus whenever \asec selects a node for expansion, the optimal path (w.r.t. $\beta(\cdot)$) to that node has been found. 
Hence, the first plan $\pi$ found by \asec has plan lower bound $c_{min}^{\Phi}(\pi)$ equal to the optimal plan lower bound w.r.t. $\Phi$, i.e., $c_{min}^{\Phi}(\pi)=c_{\Phi}^*$, implying that $\pi$ is an optimal plan w.r.t. $\Phi$.
\end{proof}

\subsection{Ontario Transportation Logistic Example}

Continuing the motivating example, we compiled a small but representative logistics transportation problem, based on real data relevant to Ontario, Canada.
The problem is to find a plan for delivery of packages in Ontario using trucks, while minimizing monetary cost, subject to the inherent uncertainty of travel times.
We next describe the technical details of the experiment.

The road system considers 9 cities with 30 road segments connecting them, where each road length is specified using a 100 meter resolution, where the data is taken from https://www.google.com/maps.
The cost function per action $a$ depends on the distance traveled $d$ and on time duration $t$, and is given by $c(a)=c_1 \times d + c_2 \times t$, where $c_1=0.56$ [USD/km] and $c_2=0.5$ [USD/minute] are the cost-per-km and cost-per-minute coefficients. $c_1$ was taken from https://carcosts.caa.ca/ and is based on data tailored to a specific set of assumptions: Ford Transit 150, 2021 model, with base crew medium sized roof slide 148WB AWD trim, commuting 30000 km annually, in Ontario province. It calculates an average total cost-per-km based on the above specifications, considering current fuel prices, depreciation and maintenance costs, license and registration fees, insurance costs and even financing. $c_2$ is based on an estimate of a truck driver employer's cost in Canada (see e.g., https://www.glassdoor.ca/Salaries/index.htm).

The domain file is similar to the aforementioned "Transport Sequential", having three action templates: drive, pick-up and drop, where the prices of the two latter actions are assumed to be exact and are based on a 20 minutes time duration.
Each drive action has two estimators: the first, which doesn't incur additional runtime, assumes nothing about the specific road segment and thus uses conservative estimates of 20 km/hour and 100 km/hour for unfavorable and favorable road conditions, respectively; and the second uses pessimistic and optimistic travel time estimates taken from https://www.google.com/maps by manual queries for a specific day and hour (Friday, 10 AM) and the fastest route.
All the required data for reproducing the experiment is contained in 3 files, corresponding to the domain, problem and estimators, attached together with the technical appendix.

We evaluated \asec on the problem for $\epsilon \in \{1, 1.1 ... , 2.4\}$, and the results are summarized in Table~\ref{table:Ontatio}.
As can be seen, the results are in full correspondence with those obtained by experimenting with the modified IPC benchmarks.
We highlight several interesting conclusions from the experiment:
\begin{itemize}
	\item \asec saved between 60\% and $\sim$100\% costly estimations compared to estimation-indifferent planning, while maintaining a 100\% success rate on feasible requirements (as $\epsilon=1,1.1$ turned out to be out of reach). 
	\item The best attainable estimates had uncertainty ratios ranging from 8.5\% to 53.8\%, demonstrating that cost uncertainty is oftentimes unavoidable. In these cases standard planning would generate inferior plans.
	\item Assuming a 1ms query time for travel time prediction (representing access time to a local database), and even in the worst result of only 60\% savings, planning time is 43 minutes instead of 107 minutes for estimation-indifferent planning. For query times of 10ms (representing a very fast access to an online database), this translates to 7 hours of planning time instead of 18 hours.
\end{itemize} 

We point out that this experiment was based on a pre-compiled file containing the required estimation data, so that during planning cost estimates were obtained by repeated access to the file, with many ground actions having the same cost estimates.
In such cases where multiple actions (or edges) have the same estimate, using a cache mechanism, in order to reduce calls to external modules, should significantly improve the results. 
The design and implementation of an effective cache is left for future research.

\end{document}